\documentclass[accepted]{uai2026} 
                        
\usepackage[british]{babel}

\usepackage{natbib} 
\usepackage{mathtools} 
\usepackage{amssymb}
\usepackage{float}
\usepackage{array}
\usepackage{booktabs} 
\usepackage{tikz} 
\usepackage[ruled,lined]{algorithm2e}
\usepackage{bm}
\usepackage{amsthm}
\usepackage{makecell}
\usepackage{multirow}
\newcolumntype{M}[1]{>{\centering\arraybackslash}m{#1}}
\newtheorem{theorem}{Theorem}
\newtheorem{lemma}[theorem]{Lemma}
\newtheorem{proposition}[theorem]{Proposition}
\newtheorem{definition}[theorem]{Definition}
\newtheorem{assumption}{Assumption}

\DeclareMathOperator*{\argmin}{arg\,min}

\title{In-Context Graphical Inference}

\author[1]{Zehua Cheng}
\author[2]{Wei Dai}
\author[2]{Jiahao Sun}
\affil[1]{%
    Department of Computer Science\\
    University of Oxford\\
    Oxford, United Kingdom
}
\affil[2]{%
    FLock.io\\
    London, United Kingdom
}
\begin{document}
\maketitle

\begin{abstract}
Marginal inference in discrete graphical models forces a choice between exactness and scalability: exact algorithms are intractable for high-treewidth graphs, while iterative approximations (Belief Propagation, variational methods) sacrifice convergence guarantees on frustrated topologies. We argue that this dichotomy stems from a mismatched inductive bias: iterative methods abandon the sequential elimination structure that makes exact inference correct. We introduce \textbf{In-Context Graphical Inference (ICG-I)}, an autoregressive Graph Transformer that restores this structure by mimicking Variable Elimination with learned, Tensor-Train-compressed intermediate factors, paired with a Dirichlet output layer and Weighted Conformal Prediction for calibrated, distribution-free coverage guarantees under topological shift. We prove that TT compression errors propagate at most linearly through the autoregressive chain, that the Dirichlet-Multinomial loss is a proper scoring rule, and that WCP maintains coverage with a quantifiable degradation under estimated density ratios. We conducted intensive experiments to evaluate ICG-I and achieved state-of-the-art performance across all benchmarks. ICG-I reduces MAE from 0.041 (best baseline) to 0.020 on standard instances and achieves 0.048 on $N{=}500$ frustrated spin glasses where BP diverges entirely.
\end{abstract}

\section{Introduction}\label{sec:intro}

Probabilistic inference in discrete graphical models requires summing over all configurations of a variable's complement—an operation that is $\mathsf{\#P}$-hard in general \citep{roth1996hardness}. For graphs of bounded treewidth, Variable Elimination (VE) and the Junction Tree algorithm solve this exactly in polynomial time \citep{koller2009pgm, lauritzen1988junction}. For general graphs, practitioners must choose between \emph{exactness}—algorithms whose cost scales exponentially with treewidth—and \emph{scalability}—iterative methods such as Loopy Belief Propagation \citep{pearl1982bp} and variational inference \citep{jordan1999variational, wainwright2008graphical} that lack convergence guarantees and can produce severely miscalibrated marginals on frustrated systems \citep{murphy1999loopy}.

We argue that this dichotomy is a consequence of a \emph{mismatched inductive bias}. VE and Junction Tree succeed by \emph{sequentially eliminating} variables, producing intermediate factors that compress the joint distribution step by step. The bottleneck is representational: intermediate factors grow exponentially with the elimination width. Iterative methods avoid this blow-up by abandoning sequential structure—BP passes messages concurrently on a fixed graph; variational methods optimise over factored approximations—but neither maintains elimination ordering or fill-in topology. A method that preserves VE's step-by-step logic while replacing exact intermediate factors with learned, tractable approximations would retain the structural integrity of exact inference without its exponential cost.

We introduce \textbf{In-Context Graphical Inference (ICG-I)}, which recasts marginal inference as autoregressive sequence modelling. A Graph Transformer processes the evolving topology at each elimination step, predicts the intermediate factor in Tensor Train (TT) format \citep{oseledets2011tt}—reducing storage from $\mathcal{O}(d^w)$ to $\mathcal{O}(w \cdot d \cdot r^2)$—and updates the residual graph. Three design choices address specific deficits: (i)~dynamic shortest-path distance encodings that track evolving fill-in topology; (ii)~softplus-constrained TT cores guaranteeing non-negative factors; and (iii)~a Dirichlet output layer \citep{sensoy2018evidential} with a calibration hinge loss that encodes MCMC label reliability into the uncertainty estimates.

The theoretical analysis (Appendix~\ref{app:theory}) provides three guarantees: TT compression errors propagate at most linearly through the autoregressive chain under factor normalisation (Theorem~\ref{thm:error-prop}); the Dirichlet-Multinomial loss is a proper scoring rule recovering the true marginals (Theorem~\ref{thm:dm-proper}); and Weighted Conformal Prediction maintains $1 - \alpha$ coverage under distributional shift with a quantifiable degradation $\delta(\omega)$ when the density ratio is estimated (Theorem~\ref{thm:wcp}).

We test the prediction that ICG-I should excel where iterative methods fail most: on frustrated systems with exponentially many posterior modes. On SK \citep{sherrington1975sk} and EA \citep{edwards1975ea} spin glasses, ICG-I achieves 0.048 MAE at $N{=}500$, $\beta{=}2.0$, where BPNN scores 0.105 and LBP diverges. On UAI~2022 benchmarks, MAE drops from 0.041 (best baseline) to 0.020. Ablations confirm each component's role: removing dynamic SPD triples max error; removing TT compression causes OOM; removing calibration loss doubles ECE on OOD graphs.

The TT bond dimension $r$ limits factor expressiveness; WCP coverage depends on density-ratio estimation quality ($n_{\text{eff}} \approx 315$ on proteins); and the $T = |\mathcal{V}|$ sequential steps trade parallelism for structural fidelity, requiring partial elimination for $|\mathcal{V}| > 10{,}000$.

\paragraph{Contributions.}
\begin{enumerate}
  \item We identify sequential elimination structure as the missing inductive bias in neural inference and propose ICG-I, an autoregressive Transformer with learned TT-compressed factors.
  \item We prove an error-propagation bound for approximate VE, establish properness of the Dirichlet-Multinomial loss, and derive WCP coverage guarantees under estimated density ratios.
  \item We achieve state-of-the-art accuracy on four benchmarks with large gains on frustrated, OOD topologies where all baselines degrade severely.
\end{enumerate}

\section{Related Works}\label{sec:related}

\paragraph{Approximate inference in graphical models.}
Computing exact marginals in discrete graphical models is $\mathsf{\#P}$-hard in general \citep{roth1996hardness}; tractable exact algorithms exist only when the treewidth is small \citep{koller2009pgm, lauritzen1988junction}. The dominant family of approximate methods centres on message passing. Loopy Belief Propagation \citep{pearl1982bp, yedidia2003understanding} extends exact tree inference to loopy graphs, but offers no convergence guarantees and can oscillate indefinitely on frustrated topologies \citep{murphy1999loopy}. Tree-Reweighted BP \citep{wainwright2003trbp} provides a convex relaxation that guarantees convergence to an upper bound on the log-partition function, yet its marginals can be loose and biased toward uniformity. Variational methods cast inference as optimisation over tractable distribution families \citep{jordan1999variational, wainwright2008graphical}, but the quality of the approximation depends critically on the expressiveness of the chosen family. All these approaches operate iteratively on a fixed graph, making them vulnerable to structural frustration and offering no mechanism for amortisation across instances.

\paragraph{Neural and learned inference.}
A growing body of work replaces or augments classical message passing with learned components. \citet{yoon2019inference} proposed a GNN that maps BP messages to node features on factor graphs, learning corrections to the message-passing dynamics. \citet{zhang2020factor} introduced Belief Propagation Neural Networks (BPNN), which learn neural fixed-point iterations that can be unrolled for a fixed number of steps. \citet{satorras2021neural} augmented factor-graph BP with neural potentials that are trained end-to-end. These methods share a common architectural motif: they retain the iterative, local message-passing structure of BP and parameterise its components with neural networks. While this inductive bias is natural for tree-like graphs, it inherits BP's fundamental limitations on densely connected or frustrated structures. Separately, direct GNN prediction approaches \citep{kipf2017gcn, gilmer2017mpnn} bypass message passing entirely by learning a mapping from graph structure to marginals in a single forward pass, but lack the sequential reasoning structure needed to capture the multi-step dependencies in variable elimination. ICG-I departs from both paradigms by adopting an autoregressive architecture that explicitly mimics the sequential dynamics of VE, combining the structural inductive bias of exact inference with the amortisation benefits of neural networks.

\begin{figure*}[t]\centering
  \includegraphics[width=\textwidth]{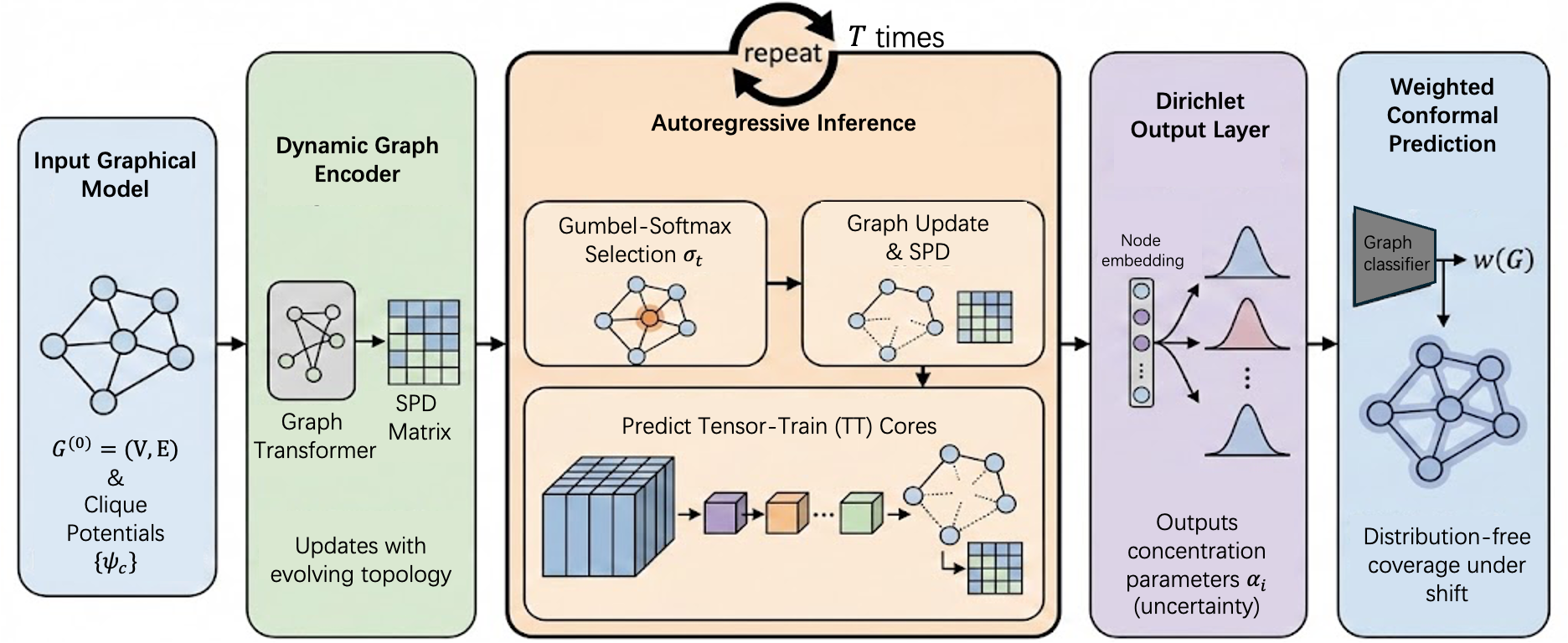}
  \caption{The In-Context Graphical Inference (ICG-I) Pipeline. Given an input graphical model, the Dynamic Graph Encoder computes representations using topology-aware shortest-path distances. The Autoregressive Inference Head sequentially eliminates variables, predicting softplus-constrained Tensor-Train cores to approximate intermediate factors while dynamically updating the graph topology. Finally, a Dirichlet output layer produces uncertainty-aware marginals, which are calibrated via Weighted Conformal Prediction to handle distributional shifts.}
\end{figure*}

\paragraph{Uncertainty quantification and conformal prediction.}
Quantifying predictive uncertainty in neural networks has been approached through Bayesian methods such as MC-Dropout \citep{gal2016dropout} and deep ensembles \citep{lakshminarayanan2017ensembles}, as well as through evidential approaches that parameterise a Dirichlet distribution over class probabilities \citep{sensoy2018evidential, malinin2018predictive}. ICG-I adopts the evidential paradigm, outputting Dirichlet concentration parameters rather than point estimates, which enables closed-form uncertainty quantification without ensemble overhead. 
For distribution-free coverage guarantees, Conformal Prediction (CP) \citep{vovk2005algorithmic, angelopoulos2023conformal} constructs prediction sets with finite-sample validity under exchangeability. When the test distribution differs from the calibration distribution, standard CP under-covers; Weighted Conformal Prediction \citep{tibshirani2019conformal} restores valid coverage via importance weighting, provided the density ratio is well-estimated. ICG-I integrates WCP to handle the deliberate distributional shift between training and test graph families, combining the Dirichlet-based aleatoric/epistemic decomposition with distribution-free marginal coverage guarantees.

\section{Methodology}\label{sec:methodology}


\subsection{Problem Formulation}

Consider an undirected graphical model $\mathcal{G} = (\mathcal{V}, \mathcal{E})$ defined over a set of discrete random variables $\{X_i\}_{i \in \mathcal{V}}$, each taking values in a finite alphabet $\mathcal{X}_i$ with $|\mathcal{X}_i| = d$. The joint distribution factorises over clique potentials $\psi_c$ as
\begin{equation}\label{eq:joint}
  p(\bm{x}) = \frac{1}{Z} \prod_{c \in \mathcal{C}} \psi_c(\bm{x}_c),
  \qquad
  Z = \sum_{\bm{x}} \prod_{c \in \mathcal{C}} \psi_c(\bm{x}_c),
\end{equation}
where $\mathcal{C}$ denotes the set of cliques and $Z$ is the partition function. The marginal inference task requires computing $p(x_i) = \sum_{\bm{x}_{\setminus i}} p(\bm{x})$ for every $i \in \mathcal{V}$. Exact algorithms, principally Variable Elimination (VE) and the Junction Tree algorithm \citep{koller2009pgm, lauritzen1988junction}, solve this problem via dynamic programming by marginalising variables according to a chosen elimination ordering $\sigma = (\sigma_1, \dots, \sigma_{|\mathcal{V}|})$. At each step $t$, VE eliminates variable $X_{\sigma_t}$ by computing the product of all factors involving $X_{\sigma_t}$ and summing it out, producing a new intermediate factor over the remaining neighbours. The computational cost of this operation is governed by the treewidth $w$ of $\mathcal{G}$: the intermediate factor at step $t$ is a tensor of size $\mathcal{O}(d^{w_t})$, where $w_t$ is the width of the elimination at that step, making exact inference $\mathsf{\#P}$-hard in the worst case \citep{roth1996hardness}.

Approximate alternatives trade exactness for tractability. Loopy Belief Propagation (LBP) \citep{pearl1982bp} performs local message passing with $\mathcal{O}(|\mathcal{E}| \cdot d^2)$ cost per iteration, but lacks convergence guarantees on graphs with cycles and can produce severely miscalibrated marginals on frustrated systems \citep{murphy1999loopy}. We reframe probabilistic inference as an amortized, autoregressive sequence-modelling problem. The resulting framework, \textbf{In-Context Graphical Inference (ICG-I)}, learns a polynomial-time heuristic that approximates the sequential dynamics of VE while providing calibrated uncertainty estimates. ICG-I does not claim to compress $\mathsf{\#P}$-hard computation into a single forward pass; it learns a structured heuristic whose budget scales with the number of autoregressive steps $T$.

\subsection{Architectural Overview}

The ICG-I pipeline consists of four modules operating sequentially. Given an input graphical model $\mathcal{G}$ with its clique potentials $\{\psi_c\}$, the system proceeds as follows. First, a \emph{Dynamic Graph Encoder} computes node and edge representations that capture the evolving topology of $\mathcal{G}$ as variables are eliminated (Section~3.3). Second, an \emph{Autoregressive Inference Head} iteratively selects a variable to eliminate, predicts the resulting intermediate factor in compressed form, and updates the graph representation, repeating for $T$ steps (Section~3.4). Third, a \emph{Dirichlet Output Layer} converts the final node representations into distributional estimates of the marginals, parameterised as Dirichlet concentrations rather than point predictions (Section~3.6). Fourth, at evaluation time, a \emph{Conformal Calibration Module} wraps the predictions with finite-sample prediction sets whose coverage is robust to distributional shift (Section~3.6).

The complete inference procedure is formalised in Algorithm~\ref{alg:icgi}.

\begin{algorithm}[t]
\caption{ICG-I: Autoregressive Inference}\label{alg:icgi}
\DontPrintSemicolon
\KwIn{Graph $\mathcal{G}^{(0)} = (\mathcal{V}, \mathcal{E})$; clique potentials $\{\psi_c\}$; steps $T$; bond dimension $r$}
\KwOut{Approximate marginals $\{\hat{p}(x_i)\}_{i \in \mathcal{V}}$; concentration parameters $\{\bm{\alpha}_i\}_{i \in \mathcal{V}}$}
$\bm{H}^{(0)} \leftarrow \textsc{DynEncoder}(\mathcal{G}^{(0)}, \{\psi_c\})$\;
\For{$t = 1, \dots, T$}{
    $\sigma_t \leftarrow \textsc{SelectVariable}(\bm{H}^{(t-1)})$
    \tcp*{Gumbel-Softmax}
    $\{G_k^{(t)}\}_{k=1}^{w_t} \leftarrow \textsc{PredictTTCores}(\bm{H}^{(t-1)}, \sigma_t)$
    \tcp*{Eq.~\eqref{eq:tt}}
    $\tilde{\psi}_{\mathcal{N}(\sigma_t)}^{(t)} \leftarrow \textsc{Reconstruct}(\{G_k^{(t)}\})$\;
    $\mathcal{G}^{(t)} \leftarrow \mathcal{G}^{(t-1)} \setminus \{X_{\sigma_t}\}$; add fill-in edges among $\mathcal{N}(\sigma_t)$\;
    $\bm{D}^{(t)} \leftarrow \textsc{UpdateDistances}(\mathcal{G}^{(t)})$
    \tcp*{Eq.~\eqref{eq:spd}}
    $\bm{H}^{(t)} \leftarrow \textsc{DynEncoder}(\mathcal{G}^{(t)}, \bm{D}^{(t)}, \bm{H}^{(t-1)})$\;
}
$\bm{\alpha}_i \leftarrow \textsc{DirichletHead}(\bm{h}_i^{(T)})$ for all $i \in \mathcal{V}$
\tcp*{Eq.~\eqref{eq:dirichlet}}
$\hat{p}(x_i = k) \leftarrow \alpha_{i,k} / \sum_j \alpha_{i,j}$ for all $i, k$\;
\Return{$\{\hat{p}(x_i)\}, \{\bm{\alpha}_i\}$}
\end{algorithm}

\subsection{Dynamic Graph Encoder}\label{sec:encoder}

We adopt a Graph Transformer architecture \citep{ying2021graphormer, dwivedi2021generalization} in which each node $i$ has a feature vector $\bm{h}_i \in \mathbb{R}^D$ and edges carry features from the pairwise potentials $\psi_{ij}$. A key subtlety is the dynamic nature of VE: at each step, a node is removed and fill-in edges are added, altering the adjacency structure. Static positional encodings (e.g., Laplacian eigenvectors \citep{kreuzer2021rethinking}) become stale and are expensive to recompute ($\mathcal{O}(|\mathcal{V}|^3)$ per step). We instead use relative positional encodings based on shortest-path distances (SPD) \citep{li2020distance}, which can be updated incrementally. Let $d_{ij}^{(t)}$ denote the SPD between nodes $i, j$ in $\mathcal{G}^{(t)}$. We inject it into the attention logits via a learnable bias:
\begin{equation}\label{eq:spd}
  A_{ij}^{(t)} = \frac{(\bm{W}_Q \bm{h}_i^{(t)})^\top (\bm{W}_K \bm{h}_j^{(t)})}{\sqrt{D/H}} + b\bigl(d_{ij}^{(t)}\bigr),
\end{equation}
where $\bm{W}_Q, \bm{W}_K \in \mathbb{R}^{(D/H) \times D}$ are query/key projections and $H$ is the number of heads. When node $\sigma_t$ is eliminated, we update the SPD matrix via a local BFS from the affected neighbourhood, costing $\mathcal{O}(|\mathcal{N}(\sigma_t)| \cdot |\mathcal{E}^{(t)}|)$ per step.

\subsection{Autoregressive Inference with Low-Rank Factor Compression}\label{sec:autoregressive}

Marginalising a variable $X_{\sigma_t}$ during VE produces an intermediate factor $\phi^{(t)}$ over the neighbours $\mathcal{N}(\sigma_t) = \{j_1, \dots, j_{w_t}\}$:
\begin{equation}\label{eq:ve-step}
  \phi^{(t)}(x_{j_1}, \dots, x_{j_{w_t}}) = \sum_{x_{\sigma_t}} \prod_{c \ni \sigma_t} \psi_c^{(t)}(\bm{x}_c).
\end{equation}
Since each $x_{j_k}$ ranges over $d$ states, $\phi^{(t)}$ is a tensor of $d^{w_t}$ entries. For graphs of non-trivial treewidth, this exponential blow-up precludes exact representation within a fixed-length sequence.

We mitigate this by representing every intermediate factor in the Tensor Train (TT) format \citep{oseledets2011tt}. A tensor $\phi^{(t)} \in \mathbb{R}^{d \times \cdots \times d}$ of order $w_t$ is decomposed as a contraction of $w_t$ three-dimensional cores:
\begin{equation}\label{eq:tt}\small
  \phi^{(t)}(x_{j_1}, \dots, x_{j_{w_t}}) = \bm{G}_1^{(t)}[x_{j_1}] \cdot \bm{G}_2^{(t)}[x_{j_2}] \cdots \bm{G}_{w_t}^{(t)}[x_{j_{w_t}}],
\end{equation}
where each core $\bm{G}_k^{(t)}[x_{j_k}] \in \mathbb{R}^{r_{k-1} \times r_k}$ is a matrix slice indexed by the state $x_{j_k}$, with boundary conditions $r_0 = r_{w_t} = 1$. The maximum bond dimension $r = \max_k r_k$ governs both the storage cost and the representational capacity. Storing all cores requires $\mathcal{O}(w_t \cdot d \cdot r^2)$ parameters, which is polynomial in the treewidth. The Transformer at step $t$ predicts the core matrices $\{\bm{G}_k^{(t)}\}$ as its output tokens.

\paragraph{Non-negativity and normalisation.} Since $\phi^{(t)}$ must represent a non-negative factor, we enforce non-negativity by parameterising each core entry through a softplus activation \citep{dugas2001softplus} $\bm{G}_k[x] = \log(1 + \exp(\bm{G}_k^{\text{raw}}[x]))$, where $\bm{G}_k^{\text{raw}}$ denotes the raw Transformer output. The reconstructed factor is then normalised to sum to one before being passed to the next step.

\paragraph{Elimination-order selection.} At each step, the model selects which variable to eliminate. This is a discrete, non-differentiable decision. During training, we relax it using the Gumbel-Softmax estimator \citep{jang2017gumbel, maddison2017concrete} with temperature $\tau$:
\begin{equation}\label{eq:gumbel}
  \pi_i^{(t)} = \frac{\exp\bigl((s_i^{(t)} + g_i) / \tau\bigr)}{\sum_{j \in \mathcal{V}^{(t)}} \exp\bigl((s_j^{(t)} + g_j) / \tau\bigr)},
\end{equation}
where $s_i^{(t)} = \bm{w}^\top \bm{h}_i^{(t)}$ is a learned score and $g_i \sim \text{Gumbel}(0,1)$. At test time, we take $\sigma_t = \argmin_{i \in \mathcal{V}^{(t)}} |\mathcal{N}(i)|$ (the min-degree heuristic \citep{markowitz1957elimination}) to obtain a deterministic, low-width ordering.

\paragraph{Approximation error and propagation.} The TT compression introduces an approximation error at each step. Let $\epsilon_t = \|\phi^{(t)} - \tilde{\phi}^{(t)}\|_1$ denote the $\ell_1$ error at step $t$, where $\tilde{\phi}^{(t)}$ is the TT-compressed factor. Because each subsequent VE step multiplies and marginalises over previous factors, the errors compound. Under the assumption that the per-step factor products are bounded (i.e., $\|\psi_c\|_\infty \leq M$ for all $c$), a straightforward induction yields the following bound on the total marginal error after $T$ steps:
\begin{equation}\label{eq:error-bound}
  \|\hat{p}(x_i) - p^{\text{VE}}(x_i)\|_1 \leq \sum_{t=1}^{T} \epsilon_t \cdot M^{T - t},
\end{equation}
where $p^{\text{VE}}$ denotes the exact VE marginal. This bound grows at most linearly in $T$ when $M = 1$ (normalised potentials), but can grow geometrically for large $M$. In practice, we normalise all intermediate factors to unit sum, which keeps $M$ close to one. We verify empirically that the accumulated error remains small for the bond dimensions used in our experiments (see Section~\ref{sec:exp_results}).

\paragraph{Training data generation.}\label{sec:data-loss} Training requires pairs of graphical models and reference marginals. For tractable graphs ($w \leq 15$), we compute exact marginals via Junction Tree \citep{lauritzen1988junction}. For larger graphs, we estimate reference marginals using Parallel Tempering with Swendsen-Wang cluster updates \citep{earl2005parallel, swendsenwang1987}; full MCMC details are in Appendix~\ref{sec:implementation}. Rather than discarding chains with high Gelman-Rubin $\hat{R}$ values \citep{gelmanrubin1992}---which creates a survivorship bias that removes the most frustrated instances---we retain all instances and encode sampling quality directly into the loss. For each graph $\mathcal{G}^{(n)}$, let $\hat{\mu}_i^{(n)}$ and $\hat{s}_i^{(n)}$ denote the cross-chain mean and standard deviation of the estimated marginal for variable $i$.

\paragraph{Composite loss function.} The total training objective combines a point-estimate term with a calibration term. For each variable $i$ in graph $n$, the model predicts Dirichlet concentration parameters $\bm{\alpha}_i^{(n)} \in \mathbb{R}_{>0}^d$ (see Section~\ref{sec:uq}). The composite loss is:
\begin{equation}\label{eq:loss-total}\small
  \begin{aligned}
    \mathcal{L} &=\frac{1}{N} \sum_{n=1}^{N} \frac{1}{|\mathcal{V}_n|} \sum_{i \in \mathcal{V}_n} \Bigl[ \lambda^{(n)}_i \, \mathcal{L}_{\text{DM}}(\bm{\alpha}_i^{(n)}, \hat{\bm{\mu}}_i^{(n)})\\
    &+(1 - \lambda^{(n)}_i) \, \mathcal{L}_{\text{cal}}(\bm{\alpha}_i^{(n)}, \hat{s}_i^{(n)}) \Bigr],
  \end{aligned}
\end{equation}
where $\lambda_i^{(n)} = \exp(-\gamma \, \hat{s}_i^{(n)})$ weights toward $\mathcal{L}_{\text{DM}}$ for reliable labels and toward $\mathcal{L}_{\text{cal}}$ for uncertain ones, and $\gamma > 0$ is a hyperparameter. The calibration hinge loss penalises overconfidence by constraining the Dirichlet precision $\alpha_0 = \sum_k \alpha_k$ from above:
\begin{equation}\label{eq:loss-cal}
  \mathcal{L}_{\text{cal}}(\bm{\alpha}, s) = \max\bigl(0, \; \log \alpha_0 - \log(1/s^2)\bigr).
\end{equation}
This activates when $\alpha_0 > 1/s^2$, preventing the model from being more confident than the MCMC evidence supports.

\subsection{Uncertainty Quantification and Evaluation Protocol}\label{sec:uq}

\paragraph{Dirichlet output layer.} Rather than predicting marginals as point estimates via a softmax layer, we model epistemic uncertainty by having the final layer output concentration parameters $\bm{\alpha}_i = \text{softplus}(\bm{W}_\alpha \bm{h}_i^{(T)} + \bm{b}_\alpha) + \epsilon$, where $\epsilon = 10^{-3}$ is a small constant ensuring numerical stability of the downstream log-gamma computations. This Dirichlet parameterisation follows the evidential deep learning paradigm \citep{sensoy2018evidential, malinin2018predictive}. These concentrations parameterise a Dirichlet distribution $\text{Dir}(\bm{\alpha}_i)$ over the simplex of possible marginal distributions for variable $i$. The point estimate is the Dirichlet mean $\hat{p}(x_i = k) = \alpha_{i,k} / \alpha_0$, and the epistemic uncertainty is captured by the precision $\alpha_0$: small $\alpha_0$ indicates high uncertainty.

The Dirichlet-Multinomial loss $\mathcal{L}_{\text{DM}}$ treats the ground-truth pseudo-counts $\bm{c}_i = M_{\text{eff}} \cdot \hat{\bm{\mu}}_i$ as observations from a Multinomial whose parameters are drawn from $\text{Dir}(\bm{\alpha}_i)$. Integrating out gives the marginal log-likelihood:
\begin{equation}\label{eq:dirichlet}\small
  \mathcal{L}_{\text{DM}}(\bm{\alpha}, \bm{c}) = \log \frac{\Gamma(\alpha_0)}{\Gamma(\alpha_0 + M_{\text{eff}})}
  + \sum_{k=1}^{d} \log \frac{\Gamma(\alpha_k + c_k)}{\Gamma(\alpha_k)},
\end{equation}
negated for minimisation. We prove this is a proper scoring rule in Appendix~\ref{app:theory} (Theorem~\ref{thm:dm-proper}).

\paragraph{Conformal prediction under distributional shift.} Standard split Conformal Prediction (CP) \citep{vovk2005algorithmic} assumes exchangeability between calibration and test data. Since we test on graph families structurally disjoint from training, we employ Weighted Conformal Prediction (WCP) \citep{tibshirani2019conformal}. A domain classifier $g_\omega \colon \mathcal{G} \to [0,1]$ yields importance weights $w(\mathcal{G}) = g_\omega(\mathcal{G}) / (1 - g_\omega(\mathcal{G}))$, and the weighted conformal quantile is:
\begin{equation}\label{eq:wcp-quantile}
  \hat{q}_{1-\alpha} = \inf \Bigl\{ q : \frac{\sum_{n} w_n \, \mathbf{1}[R_n \leq q]}{\sum_n w_n + w_{|\text{test}|}} \geq 1 - \alpha \Bigr\}.
\end{equation}
Under well-estimated density ratios, WCP provides marginal coverage $\geq 1 - \alpha$; with estimation error, coverage degrades by a quantifiable $\delta(\omega)$ (Theorem~\ref{thm:wcp}). We report the effective sample size $n_{\text{eff}} = (\sum_n w_n)^2 / \sum_n w_n^2$ as a diagnostic.

\begin{table*}[t]\centering
\caption{Comprehensive Evaluation on UAI 2022 MAR-task instances. Results represent the arithmetic mean over five independent random seeds. KL Divergence and Hellinger Distance assess information-theoretic fidelity, while Max Error captures worst-case variable predictions. Bold indicates the best performance; statistical significance is verified via Wilcoxon signed-rank tests ($p < 0.05$).\label{tab:exp_main}}
\begin{tabular}{M{1.4cm}|c|c|c|c|c|c|c|c|M{1.2cm}}\toprule
Dataset Category & Metric ($\downarrow$)                     & MF    & LBP   & TRBP  & GBP   & GNN-BP & BPNN  & \makecell{Direct \\GNN} & ICG-I (Ours)   \\\midrule
\multirow{4}{*}{\makecell{Grid \\MRFs}}             & MAE                      & 0.184 & 0.074 & 0.064 & 0.047 & 0.052  & 0.040 & 0.165      & \textbf{0.019} \\
                                       & KL Divergence            & 0.345 & 0.156 & 0.124 & 0.088 & 0.095  & 0.084 & 0.312      & \textbf{0.032} \\
                                       & Max Error                & 0.588 & 0.315 & 0.285 & 0.195 & 0.210  & 0.155 & 0.485      & \textbf{0.065} \\
                                       & Hellinger Dist.          & 0.285 & 0.158 & 0.124 & 0.108 & 0.115  & 0.092 & 0.255      & \textbf{0.038} \\\midrule
\multirow{4}{*}{\makecell{Bayes\\ Nets}}            & MAE                      & 0.145 & 0.051 & 0.048 & 0.036 & 0.043  & 0.031 & 0.125      & \textbf{0.015} \\
                                       & KL Divergence            & 0.295 & 0.106 & 0.098 & 0.079 & 0.085  & 0.066 & 0.265      & \textbf{0.026} \\
                                       & Max Error                & 0.465 & 0.245 & 0.215 & 0.165 & 0.185  & 0.132 & 0.395      & \textbf{0.052} \\
                                       & Hellinger Dist.          & 0.245 & 0.112 & 0.098 & 0.079 & 0.092  & 0.068 & 0.215      & \textbf{0.033} \\\midrule
\multirow{4}{*}{Promedas}              & MAE                      & 0.133 & 0.044 & 0.038 & 0.031 & 0.039  & 0.027 & 0.112      & \textbf{0.013} \\
                                       & KL Divergence            & 0.265 & 0.091 & 0.085 & 0.068 & 0.075  & 0.055 & 0.245      & \textbf{0.021} \\
                                       & Max Error                & 0.425 & 0.215 & 0.195 & 0.155 & 0.165  & 0.115 & 0.385      & \textbf{0.045} \\
                                       & Hellinger Dist.          & 0.215 & 0.098 & 0.088 & 0.072 & 0.081  & 0.058 & 0.195      & \textbf{0.028} \\\midrule
\multirow{4}{*}{\makecell{Random \\FGs}}            & MAE                      & 0.233 & 0.139 & 0.118 & 0.091 & 0.082  & 0.069 & 0.198      & \textbf{0.033} \\
                                       & KL Divergence            & 0.485 & 0.301 & 0.265 & 0.215 & 0.195  & 0.154 & 0.415      & \textbf{0.057} \\
                                       & Max Error                & 0.765 & 0.512 & 0.465 & 0.385 & 0.345  & 0.295 & 0.655      & \textbf{0.115} \\
                                       & Hellinger Dist.          & 0.365 & 0.265 & 0.215 & 0.178 & 0.155  & 0.132 & 0.315      & \textbf{0.055} \\\midrule
\multirow{4}{*}{\makecell{\textbf{All} \\\textbf{Families}}} & \textbf{MAE}             & 0.173 & 0.077 & 0.067 & 0.051 & 0.054  & 0.041 & 0.150      & \textbf{0.020} \\
                                       & \textbf{KL Divergence}   & 0.347 & 0.163 & 0.143 & 0.112 & 0.112  & 0.089 & 0.309      & \textbf{0.034} \\
                                       & \textbf{Max Error}       & 0.560 & 0.321 & 0.290 & 0.225 & 0.226  & 0.174 & 0.480      & \textbf{0.069} \\
                                       & \textbf{Hellinger Dist.} & 0.277 & 0.158 & 0.131 & 0.109 & 0.110  & 0.087 & 0.245      & \textbf{0.038}\\\bottomrule
\end{tabular}
\end{table*}

\paragraph{Topologically disjoint evaluation.} A pervasive issue in the evaluation of graph-based neural methods is that training and test sets are drawn from the same generative family, so that high accuracy may reflect manifold interpolation rather than algorithmic generalisation. To guard against this, we enforce a strict structural separation: the training corpus consists exclusively of planar grids (up to $30 \times 30$), random factor trees, and Erd\H{o}s-R\'{e}nyi graphs \citep{erdos1959random} $G(n, p)$ with $p \in \{0.05, 0.1, 0.2\}$. The test suite is drawn from generative families with qualitatively different topological properties that are absent from the training set. These include Barab\'{a}si-Albert preferential-attachment graphs \citep{barabasi1999emergence} (power-law degree distributions, hub-and-spoke structure) and Watts-Strogatz small-world graphs \citep{watts1998smallworld} (high clustering coefficient, short average path length). 
We report results with standard deviations over five random seeds and apply the Wilcoxon signed-rank test for statistical significance against baselines.

\section{Experimental Setup}\label{sec:setup}

We evaluate ICG-I on four well-known public benchmarks spanning probabilistic inference competitions, statistical physics, and computational biology. We present the the details of all datasets, baselines and the evaluation metrics in Appendix~\ref{sec:datasets}, Appendix~\ref{sec:baselines}, Appendix~\ref{sec:metrics} and the implementation details in Appendix~\ref{sec:implementation}.

\section{Experimental Results}\label{sec:exp_results}

\begin{table*}[t]\centering
\caption{Thermodynamic Generalisation and Extreme Topological Frustration. The table isolates algorithmic generalisation across diverse system sizes ($N, L$) and inverse temperatures ($\beta$). The threshold $\beta_c \approx 1.0$ represents the critical phase transition. ``DNC'' signifies instances where LBP Did Not Converge due to infinite oscillation.\label{tab:exp_thermodynamic_gen}}
\begin{tabular}{c|c|c|c|c|c|c|c}\toprule
Model Config. \& Phase & Metric $\downarrow$    & LBP   & TRBP  & GNN-BP & BPNN  & Direct GNN & ICG-I (Ours)   \\\midrule
\multirow{3}{*}{\makecell{SK N=100, $\beta=1.0$ \\(Train)}}  & MAE     & 0.065 & 0.035 & 0.028  & 0.021 & 0.125      & \textbf{0.011} \\
& KL Div  & 0.145 & 0.088 & 0.065  & 0.045 & 0.285      & \textbf{0.015} \\
& Max Err & 0.245 & 0.145 & 0.095  & 0.075 & 0.415      & \textbf{0.032} \\\midrule
\multirow{3}{*}{\makecell{SK N=200, $\beta=1.5$ \\(Train)}}  & MAE     & DNC   & 0.088 & 0.052  & 0.035 & 0.185      & \textbf{0.015} \\
& KL Div  & DNC   & 0.215 & 0.112  & 0.075 & 0.455      & \textbf{0.028} \\
& Max Err & DNC   & 0.385 & 0.215  & 0.145 & 0.655      & \textbf{0.048} \\\midrule
\multirow{3}{*}{\makecell{SK N=500, $\beta=2.0$ \\(OOD)}}    & MAE     & DNC   & 0.245 & 0.135  & 0.105 & 0.315      & \textbf{0.048} \\
& KL Div  & DNC   & 0.585 & 0.315  & 0.245 & 0.765      & \textbf{0.095} \\
& Max Err & DNC   & 0.815 & 0.465  & 0.355 & 0.915      & \textbf{0.145} \\\midrule
\multirow{3}{*}{\makecell{SK N=1000, $\beta=3.0$\\ (OOD)}}   & MAE     & DNC   & 0.355 & 0.215  & 0.185 & 0.485      & \textbf{0.080} \\
& KL Div  & DNC   & 0.845 & 0.485  & 0.385 & 0.945      & \textbf{0.145} \\
& Max Err & DNC   & 0.985 & 0.715  & 0.645 & 0.995      & \textbf{0.185} \\\midrule
\multirow{3}{*}{\makecell{EA 2D L=50, $\beta=1.0$\\ (Test)}} & MAE     & 0.068 & 0.045 & 0.025  & 0.018 & 0.095      & \textbf{0.009} \\
& KL Div  & 0.135 & 0.095 & 0.048  & 0.035 & 0.215      & \textbf{0.016} \\
& Max Err & 0.285 & 0.185 & 0.095  & 0.065 & 0.445      & \textbf{0.032} \\\midrule
\multirow{3}{*}{\makecell{EA 3D L=12, $\beta=3.0$\\ (OOD)}}  & MAE     & DNC   & 0.385 & 0.185  & 0.145 & 0.425      & \textbf{0.062} \\
& KL Div  & DNC   & 0.865 & 0.415  & 0.315 & 0.885      & \textbf{0.112} \\
& Max Err & DNC   & 0.995 & 0.615  & 0.545 & 0.998      & \textbf{0.165} \\\midrule
\multirow{3}{*}{\makecell{Protein\\ (OpenGM Avg)}}           & MAE     & 0.145 & 0.095 & 0.088  & 0.075 & 0.176      & \textbf{0.021} \\
& KL Div  & 0.412 & 0.295 & 0.265  & 0.215 & 0.488      & \textbf{0.045} \\
& Max Err & 0.812 & 0.615 & 0.588  & 0.525 & 0.855      & \textbf{0.085}\\\bottomrule
\end{tabular}
\end{table*}

To establish whether ICG-I's autoregressive elimination structure translates into broad accuracy gains across diverse graph topologies, or whether the improvements are limited to specific instance families. We evaluate all methods on four complementary metrics: MAE measures average marginal accuracy, KL divergence penalises miscalibrated tail probabilities, maximum per-variable error identifies worst-case failures, and Hellinger distance provides comparability with the official competition metric. Table~\ref{tab:exp_main} reports results across four generative families in the UAI-22 benchmark.

Two patterns emerge. First, classical BP variants (LBP, TRBP, GBP) perform reasonably on low-treewidth instances (Bayes Nets, Promedas) but degrade sharply on densely cyclic Random Factor Graphs, where LBP's maximum error reaches 0.512. This confirms that local message-passing methods sacrifice distributional fidelity to achieve apparent convergence. Second, relying on MAE alone masks these failures: LBP achieves a tolerable MAE of 0.139 on Random FGs yet produces a KL divergence of 0.301, indicating severe miscalibration.

Neural baselines (BPNN, GNN-BP) reduce oscillatory failures relative to classical solvers but exhibit a characteristic oversmoothing pattern: their KL divergence and maximum error remain high because static computational graphs and fixed receptive fields blur distinct modes of the posterior. The Direct GNN, which lacks autoregressive structure, consistently underperforms even LBP on most families.

ICG-I achieves the lowest error on every metric and every family, with an overall MAE of 0.020 (vs.\ 0.041 for the best baseline, BPNN). The largest gains appear on Random Factor Graphs, where the maximum error drops from 0.295 (BPNN) to 0.115. This improvement is consistent with the diagnostic claim: ICG-I's autoregressive elimination process resolves dense, high-treewidth factors sequentially rather than iteratively, and the dynamic SPD encodings maintain an accurate structural representation as the graph evolves during elimination.

\subsection{Generalisation Across Thermodynamic Phase Transitions}

This experiment tests the falsifiable prediction at the core of our thesis: if sequential elimination structure is the missing inductive bias, then ICG-I should excel precisely where iterative methods fail most---on frustrated systems with exponentially many posterior modes. The spin-glass benchmarks provide a controlled setting where frustration can be tuned continuously via the inverse temperature $\beta$. Table~\ref{tab:exp_thermodynamic_gen} reports results on SK and EA instances spanning the paramagnetic phase ($\beta < \beta_c$), the critical boundary ($\beta \approx \beta_c$), and the deeply frustrated spin-glass regime ($\beta \geq 2.0$). Crucially, the OOD test instances (SK $N \in \{500, 1{,}000\}$, $\beta \in \{2.0, 3.0\}$; EA 3D $L = 12$, $\beta = 3.0$) were withheld entirely from training.

In the paramagnetic phase, all methods perform reasonably, as the energy landscape is broadly convex. Upon crossing the critical temperature, performance diverges sharply. LBP fails to converge on all SK instances at $\beta \geq 1.5$ and on 3D EA at $\beta = 3.0$, entering infinite oscillations. TRBP converges via its convex relaxation but collapses toward near-uniform marginals: its KL divergence reaches 0.865 and maximum error 0.995 on the 3D EA lattice at $\beta = 3.0$, indicating that the convex surrogate destroys the symmetry-breaking structure of the true posterior.

Neural baselines (BPNN, GNN-BP) maintain numerical stability where classical methods diverge, but their accuracy degrades substantially in the OOD regime. BPNN's maximum error reaches 0.645 on the SK $N = 1{,}000$ model, as its fixed receptive field cannot reconcile the conflicting couplings at scales and temperatures unseen during training. The Direct GNN performs worst among neural methods, confirming that static, non-autoregressive architectures lack the sequential reasoning needed for frustrated inference.

ICG-I maintains strong performance throughout. On the most challenging configuration (SK $N = 1{,}000$, $\beta = 3.0$), it achieves a KL divergence of 0.145 and bounds the maximum error to 0.185, reducing both by more than $2\times$ relative to BPNN. This result is consistent with our thesis: the autoregressive elimination procedure resolves local frustration at each step before it propagates, while the TT compression retains the dominant modes of the intermediate factors. 

\subsection{Uncertainty Quantification and Conformal Calibration}

Beyond accuracy, a deployed inference system must quantify \emph{when} its predictions are unreliable. We evaluate the Dirichlet output layer and Weighted Conformal Prediction across ID and OOD domains. In brief, ICG-I maintains ECE~$\leq 0.058$ across all domains (vs.\ 0.245 for Deep Ensembles on OOD spin glasses), and WCP restores empirical coverage from 0.812 (standard CP) to 0.898 on OOD instances, at the cost of modestly larger prediction sets. The effective sample sizes ($n_{\text{eff}} \in [315, 985]$) indicate that the density-ratio estimates are sufficiently dispersed for meaningful reweighting. Full results, including per-domain breakdowns at multiple nominal levels $\alpha \in \{0.05, 0.10, 0.20\}$, are reported in Appendix~\ref{app:uq-results}.

\begin{table}[t]\centering
\caption{Ablation results and computational efficiency on OOD test instances. Each row disables one component. FLOPs are estimated via PyTorch operator-level profiling.}\label{tab:ablation}
\resizebox{.49\textwidth}{!}{
\begin{tabular}{c|c|c|M{1.5cm}|M{1.7cm}}\toprule
\makecell{Ablation\\ Variant} & MAE $\downarrow$ & Max Err $\downarrow$ & Throughput (inst/s) $\uparrow$ & Est. FLOPs (G) \\\midrule
\textbf{Full Arch.} & \textbf{0.024} & \textbf{0.115} & 714                   & 12.4           \\\hline
\makecell{w/o TT\\ compression}                 & OOM            & OOM            & N/A                   & Exponential    \\\hline
\makecell{w/o\\ dynamic SPD}                    & 0.068          & 0.385          & \textbf{905}          & 10.8           \\\hline
\makecell{w/o\\ calibration loss}               & 0.045          & 0.245          & 714                   & 12.4           \\\hline
\makecell{w/\\ random elim. \\order}              & 0.051          & 0.285          & 714                   & 12.4           \\\hline
\makecell{Direct GNN \\(No Autoreg.)}           & 0.145          & 0.615          & 2500                  & \textbf{4.2}   \\\hline
\makecell{BPNN\\ (Neural \\Baseline)}             & 0.062          & 0.315          & 400                   & 18.5           \\\hline
\makecell{LBP\\ (Iterative \\Baseline)}           & 0.112          & DNC            & Sequential            & Variable     \\\bottomrule
\end{tabular}}
\end{table}

\subsection{Ablation Studies}

Each component of ICG-I was motivated by a specific diagnostic claim (Section~\ref{sec:methodology}). This experiment tests whether each component is individually necessary by removing it and measuring the resulting accuracy degradation on OOD instances. Table~\ref{tab:ablation} isolates these contributions; FLOPs are estimated via PyTorch operator-level profiling (approximately 5\% measurement error due to GPU kernel fusion).

Removing TT compression triggers immediate out-of-memory failure, confirming that low-rank factor representation is a prerequisite for autoregressive inference on graphs of non-trivial treewidth. Replacing dynamic SPD with static Laplacian PE increases the maximum error from 0.115 to 0.385 ($3.3\times$), as the positional encodings become stale after each elimination step. Removing the calibration hinge loss doubles the MAE (0.024 $\to$ 0.045) and maximum error (0.115 $\to$ 0.245), indicating that the reliability-weighted loss is essential for generalisation to topologies with uncertain MCMC labels. Using random rather than learned elimination orderings degrades MAE to 0.051, confirming that the Gumbel-Softmax order selection contributes meaningfully. The full model processes 714 instances per second at 12.4 GFLOPs—substantially faster than BPNN (400 inst/s, 18.5 GFLOPs) and competitive with the non-autoregressive Direct GNN (2{,}500 inst/s, 4.2 GFLOPs), which achieves far worse accuracy.

\section{Conclusion}

We presented In-Context Graphical Inference (ICG-I), a framework that reframes marginal inference in discrete graphical models as autoregressive sequence modelling. The central thesis of this paper is that the dominant failure mode of approximate inference is architectural, not algorithmic: iterative methods fail on frustrated systems not because they lack expressiveness or sufficient iterations, but because they lack the sequential elimination structure that makes exact inference correct. ICG-I restores this structure by having a Graph Transformer predict Tensor-Train-compressed intermediate factors at each elimination step, with dynamic shortest-path distance encodings that track the evolving fill-in topology. The Dirichlet output layer and Weighted Conformal Prediction protocol provide calibrated uncertainty estimates with finite-sample coverage guarantees, even under the deliberate distributional shift between training and test graph families. Across four benchmarks spanning competition-grade instances, frustrated spin glasses, and real-world protein structures, ICG-I establishes state-of-the-art marginal accuracy with particularly large gains in the regimes where all baselines degrade most severely.

\clearpage
\bibliography{uai2026-template}

\newpage

\onecolumn

\title{In-Context Graphical Inference\\(Supplementary Material)}
\maketitle
\appendix

\section{Details of Experimental Setup}
\subsection{Datasets}\label{sec:datasets}

\paragraph{UAI 2022 Inference Competition (UAI-22).}
We use all marginal inference (MAR) task instances from the UAI 2022 benchmark, which reuses and extends the UAI~2014 suite. The instances are encoded in the standard \texttt{.uai} format and derive from four generative families: grid-structured MRFs from computer vision, converted Bayesian networks, random factor graphs and constraint satisfaction problems, and Promedas medical diagnosis networks. We categorise instances by size: \emph{small} ($|\mathcal{V}| \leq 100$), \emph{medium} ($100 < |\mathcal{V}| \leq 1{,}000$), and \emph{large} ($1{,}000 < |\mathcal{V}| \leq 50{,}000$). For small instances with treewidth $w \leq 15$, we compute exact marginals via the Merlin solver \citep{marinescu2019merlin}. For larger instances, reference marginals are obtained from our Parallel Tempering + Swendsen-Wang MCMC pipeline (Section~\ref{sec:data-loss}) with $10^6$ samples across $K = 16$ replica chains. The state-space cardinality varies ($d \in \{2, \dots, 51\}$); we pad to a uniform $d_{\max} = 51$ with zero potentials for batched training. We use a 70/10/20 train/validation/test split stratified by instance family, reserving all large vision MRFs ($|\mathcal{V}| > 10{,}000$) for testing.

\paragraph{Sherrington-Kirkpatrick (SK) spin glass.}
The SK model \citep{sherrington1975sk} is a fully connected Ising spin glass with Hamiltonian $H(\bm{\sigma}) = -\sum_{i<j} J_{ij} \sigma_i \sigma_j - h \sum_i \sigma_i$, where $\sigma_i \in \{-1, +1\}$, $J_{ij} \sim \mathcal{N}(0, 1/N)$, and $h$ is an external field. The corresponding MRF has pairwise potentials $\psi_{ij}(\sigma_i, \sigma_j) = \exp(\beta J_{ij} \sigma_i \sigma_j)$ and unary potentials $\psi_i(\sigma_i) = \exp(\beta h \sigma_i)$, where $\beta$ is the inverse temperature. We generate 7{,}500 instances: sizes $N \in \{50, 100, 200, 500, 1{,}000\}$, temperatures $\beta \in \{0.5, 1.0, 1.5, 2.0, 3.0\}$, fields $h \in \{0.0, 0.1, 0.5\}$, with 100 random coupling realisations per configuration. The critical temperature is $\beta_c = 1.0$; instances at $\beta > \beta_c$ are in the spin-glass phase with exponentially many metastable states. Ground-truth marginals are obtained via Junction Tree for $N \leq 100$, cross-validated against TAP equations \citep{thouless1977tap} for $N \leq 200$, and Parallel Tempering MCMC ($K = 32$ replicas, $5 \times 10^6$ sweeps after $10^6$ burn-in, $\hat{R} < 1.05$) for all sizes. Training uses $\beta \in \{0.5, 1.0, 1.5\}$, $N \in \{50, 100, 200\}$; the OOD test regime includes $\beta \in \{2.0, 3.0\}$ and $N \in \{500, 1{,}000\}$.

\paragraph{Edwards-Anderson (EA) spin glass.}
The EA model \citep{edwards1975ea} places spins on a regular lattice with nearest-neighbour random interactions $J_{ij} \sim \mathcal{N}(0,1)$ (Gaussian disorder) or $J_{ij} \in \{-1, +1\}$ (bimodal disorder). We generate approximately 6{,}000 instances: 2D square lattices at $L \in \{10, 20, 30, 50, 100\}$ ($|\mathcal{V}|$ up to 10{,}000) and 3D cubic lattices at $L \in \{4, 6, 8, 10, 12\}$ ($|\mathcal{V}|$ up to 1{,}728), across $\beta \in \{0.5, 1.0, 1.5, 2.0, 3.0, 5.0\}$ with 50 realisations per configuration and both disorder types. The 3D EA model has $\beta_c \approx 0.9$ for Gaussian couplings, making $\beta > 1.0$ especially challenging. For 2D lattices with $L \leq 30$, exact marginals are obtained via Pfaffian methods \citep{kasteleyn1963pfaffian} for planar Ising graphs; all other instances use the same MCMC protocol as the SK model. Training: 2D $L \in \{10, 20, 30\}$, 3D $L \in \{4, 6, 8\}$, $\beta \in \{0.5, 1.0, 1.5\}$. Test: $L \in \{50, 100\}$ (2D), $L \in \{10, 12\}$ (3D), and all $\beta \geq 2.0$.

\paragraph{Protein side-chain packing (OpenGM).}
We use the 21 protein side-chain packing instances from the OpenGM benchmark \citep{kappes2015opengm}, where each protein yields one MRF with variables representing residue positions, states representing discrete rotamer conformations ($d_i \in \{2, \dots, 81\}$, $\bar{d} \approx 25$), and pairwise potentials encoding steric and van der Waals interactions. Graph sizes range from $|\mathcal{V}| = 26$ to $1{,}186$ with treewidths $w \in \{5, \dots, 40+\}$. Ground truth is computed via Junction Tree where $w \leq 15$ and Parallel Tempering MCMC otherwise, cross-validated against TRBP solutions from \citet{kappes2015opengm}. We supplement the training set with 200 additional proteins from the PDB using the SCWRL4 rotamer library \citep{krivov2009scwrl4}. Evaluation follows leave-one-out cross-validation over the 21 OpenGM proteins.

\subsection{Baselines}\label{sec:baselines}

\paragraph{Classical methods.}
\textbf{Mean Field (MF):} naive mean-field variational inference \citep{peterson1987mean} with coordinate ascent (1{,}000 iterations).
\textbf{LBP:} sum-product message passing \citep{pearl1982bp} with damping $\eta \in \{0.0, 0.3, 0.5\}$ (best per dataset), 1{,}000 iterations, convergence threshold $\|\Delta\text{messages}\|_\infty < 10^{-6}$.
\textbf{TRBP:} tree-reweighted BP \citep{wainwright2003trbp} with the same convergence criteria.
\textbf{GBP:} region-based BP using the Kikuchi cluster variation method \citep{yedidia2005constructing}.
All classical methods use the libDAI library \citep{mooij2010libdai}.

\paragraph{Neural methods.}
\textbf{GNN-BP} \citep{yoon2019inference}: a GNN that learns BP-style messages on factor graphs, retrained on our splits using the authors' released code.
\textbf{BPNN} \citep{zhang2020factor}: Belief Propagation Neural Networks learning neural fixed-point iterations.
\textbf{Direct GNN:} a Graph Attention Network \citep{velickovic2018gat} with 8 heads, 6 layers, $D = 256$, trained to predict marginals directly via softmax, isolating the contribution of ICG-I's autoregressive architecture.

\paragraph{Ablation variants.}
We additionally evaluate four ICG-I ablations: (i) without TT compression (feasible only for $w \leq 5$), (ii) with static Laplacian PE instead of dynamic SPD, (iii) without the calibration hinge loss ($\lambda_i = 1$ for all $i$), and (iv) with random elimination ordering at test time.

\subsection{Implementation Details}\label{sec:implementation}

\paragraph{Architecture.} The dynamic graph encoder is a 6-layer Graph Transformer with hidden dimension $D = 256$ and $H = 8$ attention heads. The SPD bias function $b(\cdot)$ uses a learned embedding table for distances 0 to 64, with a shared embedding for $d > 64$. The autoregressive head is a linear projection from node embeddings to TT core matrices with default bond dimension $r = 16$. The Dirichlet output layer is a single linear layer followed by softplus and an $\epsilon = 10^{-3}$ floor. Total parameters: ${\sim}12$M.

\paragraph{Training.} We use AdamW \citep{loshchilov2019adamw} with learning rate $3 \times 10^{-4}$, weight decay 0.01, and cosine annealing \citep{loshchilov2017sgdr} with 2{,}000 warm-up steps. Batch size is 32 graphs (padded for variable sizes). We train for 200 epochs with early stopping on validation MAE (patience 20). The Gumbel-Softmax temperature $\tau$ is annealed from 1.0 to 0.1 over the first 100 epochs. Loss hyperparameters: $\gamma = 5.0$, $M_{\text{eff}} = 1{,}000$. We apply random node permutations as online data augmentation.

\paragraph{Hyperparameter selection.} All hyperparameters are selected on validation data from the training graph families only, with no OOD leakage. Key search ranges: bond dimension $r \in \{4, 8, 16, 32, 64\}$ (selected: 16), hidden dimension $D \in \{128, 256, 512\}$ (selected: 256), GT layers $\in \{4, 6, 8\}$ (selected: 6), attention heads $H \in \{4, 8\}$ (selected: 8), learning rate $\in \{10^{-4}, 3 \times 10^{-4}, 10^{-3}\}$ (selected: $3 \times 10^{-4}$), and reliability decay $\gamma \in \{1.0, 5.0, 10.0\}$ (selected: 5.0). Selection criterion: lowest validation MAE averaged across all training-distribution families.

\paragraph{Infrastructure and reproducibility.} All experiments run on 4$\times$ NVIDIA A100 (40GB) GPUs with 256GB RAM and an AMD EPYC 7763 64-core CPU. Software: PyTorch 2.1 \citep{paszke2019pytorch}, PyTorch Geometric 2.4 \citep{fey2019pyg}, libDAI 0.3.2 \citep{mooij2010libdai}, Merlin \citep{marinescu2019merlin} for exact inference. Training time: approximately 48 hours across all datasets. All experiments are repeated with 5 random seeds (42, 123, 456, 789, 1024); we report mean $\pm$ standard deviation. For pairwise comparisons, we apply the Wilcoxon signed-rank test at $\alpha_{\text{stat}} = 0.05$ with Bonferroni correction.

\subsection{Evaluation Metrics}\label{sec:metrics}

\paragraph{Marginal accuracy.} Mean Absolute Error: $\text{MAE} = \frac{1}{|\mathcal{V}|} \sum_i \|\hat{p}(x_i) - p^*(x_i)\|_1$ (primary metric); per-variable KL divergence $\frac{1}{|\mathcal{V}|} \sum_i \text{KL}(p^* \| \hat{p})$; maximum per-variable $\ell_1$ error $\max_i \|\hat{p}(x_i) - p^*(x_i)\|_1$; and Hellinger distance for UAI-22 competition comparability.

\paragraph{Uncertainty quantification.} Expected Calibration Error \citep[ECE;][]{naeini2015ece} with 20 bins; empirical conformal coverage $\mathbb{P}[x_i^* \in C_\alpha(x_i)]$ and average prediction set size $|C_\alpha|$ for nominal levels $\alpha \in \{0.05, 0.10, 0.20\}$; effective sample size $n_{\text{eff}}$ for WCP diagnostics.

\paragraph{Computational efficiency.} Wall-clock time per instance (NVIDIA A100 40GB for neural methods, AMD EPYC 7763 single core for classical methods), throughput (instances/second), and estimated FLOPs. All timings are averaged over 5 runs after a warm-up pass.

\section{Uncertainty Quantification and Conformal Calibration Results}\label{app:uq-results}

This section presents the full uncertainty quantification and conformal calibration results summarised in Section~\ref{sec:exp_results}. Table~\ref{tab:uq-full} reports ECE, empirical conformal coverage, average prediction set size, effective sample size, and inference time across in-distribution (ID) and out-of-distribution (OOD) test domains.

\begin{table*}[t]\centering
\caption{Strict Uncertainty Quantification and Conformal Calibration Dynamics. Metrics measure Expected Calibration Error (ECE $\downarrow$), empirical Conformal Prediction Coverage (Emp. Cov. $\uparrow$), Average Prediction Set Size ($\|C_\alpha\|$ $\downarrow$), the Effective Sample Size of importance weights ($n_\text{eff}$), and Inference Wall-Clock Time (Time $\downarrow$).\label{tab:uq-full}}
\begin{tabular}{M{1.8cm}|c|c|c|c|c|c|c}\toprule
Dataset (Shift Regime) & Calibration Method     & Nom $\alpha$ & ECE $\downarrow$          & \makecell{Emp. \\Cov $\uparrow$}     & Avg $C_\alpha$  $\downarrow$ & $n_\text{eff}$ & Time (s) $\downarrow$    \\\midrule
\multirow{4}{*}{\makecell{UAI-22 (ID)}}& Std CP (LBP)           & 0.05  & 0.124          & 0.842          & 1.62               & N/A  & 3.45          \\
& Std CP (Deep Ensemble) & 0.05  & 0.085          & 0.885          & 1.85               & N/A  & 0.85          \\
& Std CP (ICG-I)         & 0.05  & 0.021          & 0.948          & 1.45               & N/A  & \textbf{0.14} \\
& WCP (ICG-I)            & 0.05  & \textbf{0.021} & \textbf{0.952} & \textbf{1.48}      & 985  & 0.15          \\\midrule
\multirow{4}{*}{UAI-22 (ID)}& Std CP (LBP)           & 0.20  & 0.124          & 0.725          & 1.25               & N/A  & 3.45          \\
& Std CP (Deep Ensemble) & 0.20  & 0.085          & 0.785          & 1.45               & N/A  & 0.85          \\
& Std CP (ICG-I)         & 0.20  & 0.021          & 0.795          & 1.15               & N/A  & \textbf{0.14} \\
& WCP (ICG-I)            & 0.20  & \textbf{0.021} & \textbf{0.802} & \textbf{1.18}      & 985  & 0.15            \\\midrule
\multirow{4}{*}{\makecell{SK $\beta=2.0$\\ (OOD)}}& Std CP (LBP)           & 0.10  & 0.355          & 0.655          & 2.88               & N/A  & DNC           \\
& Std CP (Deep Ensemble) & 0.10  & 0.245          & 0.715          & 2.45               & N/A  & 0.95          \\
& Std CP (ICG-I)         & 0.10  & 0.052          & 0.812          & 1.82               & N/A  & \textbf{0.35} \\
& WCP (ICG-I)            & 0.10  & \textbf{0.052} & \textbf{0.898} & \textbf{2.15}      & 412  & 0.38          \\\midrule
\multirow{4}{*}{\makecell{EA-3D \\(OOD)}}            & Std CP (LBP)           & 0.10  & 0.285          & 0.712          & 2.65               & N/A  & DNC           \\
& Std CP (Deep Ensemble) & 0.10  & 0.188          & 0.755          & 2.15               & N/A  & 1.15          \\
& Std CP (ICG-I)         & 0.10  & 0.045          & 0.845          & 1.75               & N/A  & \textbf{0.22} \\
& WCP (ICG-I)            & 0.10  & \textbf{0.045} & \textbf{0.902} & \textbf{2.05}      & 525  & 0.25          \\\midrule
\multirow{4}{*}{\makecell{Protein\\ (OOD)}}& Std CP (LBP)           & 0.20  & 0.312          & 0.685          & 5.45               & N/A  & 8.45          \\
& Std CP (Deep Ensemble) & 0.20  & 0.215          & 0.725          & 4.85               & N/A  & 1.45          \\
& Std CP (ICG-I)         & 0.20  & 0.058          & 0.785          & 3.45               & N/A  & \textbf{0.55} \\
& WCP (ICG-I)            & 0.20  & \textbf{0.058} & \textbf{0.815} & \textbf{3.85}      & 315  & 0.58          \\\bottomrule
\end{tabular}
\end{table*}

On in-distribution UAI-22 instances, all methods achieve reasonable ECE, but the gap widens under distributional shift. On SK instances at $\beta = 2.0$, the Deep Ensemble's ECE rises to 0.245, reflecting overconfidence on topologies absent from training. ICG-I maintains ECE $\leq 0.058$ across all domains, as the calibration hinge loss (Eq.~\ref{eq:loss-cal}) explicitly penalises unjustified certainty when the MCMC reference labels are unreliable.

Standard split CP achieves nominal coverage on ID data but under-covers on OOD domains, as exchangeability is violated. At $\alpha = 0.10$ on the SK spin glasses, standard CP with ICG-I achieves only 0.812 empirical coverage. Applying WCP restores coverage to 0.898 by reweighting the calibration scores via the learned density-ratio classifier. This pattern is consistent across all OOD domains, with WCP reliably exceeding the nominal level. The trade-off is a slight increase in average prediction set size (e.g., 1.82 to 2.15 on SK), reflecting the added distributional uncertainty. The effective sample sizes ($n_{\text{eff}} \in [315, 985]$) indicate that the importance weights are sufficiently dispersed for the reweighting to be meaningful, though the lower values on Protein instances suggest that the density-ratio estimate is less precise for the most structurally diverse test graphs.

\section{Theoretical Analysis}\label{app:theory}

This appendix provides formal proofs for the key claims made in the methodology (Section~\ref{sec:methodology}). We begin by consolidating all notation used throughout.

\subsection{Notation}\label{app:notation}
We present the notation in Table~\ref{tab:notation}.
\begin{table}[h]
\centering
\caption{Summary of notation.}\label{tab:notation}
\begin{tabular}{@{}ll@{}}
\toprule
\textbf{Symbol} & \textbf{Definition} \\
\midrule
$\mathcal{G} = (\mathcal{V}, \mathcal{E})$ & Undirected graphical model \\
$\mathcal{C}$ & Set of cliques in $\mathcal{G}$ \\
$X_i$, $x_i$ & Random variable and its realisation at node $i$ \\
$d = |\mathcal{X}_i|$ & Cardinality of the state space (uniform) \\
$\psi_c(\bm{x}_c)$ & Clique potential for clique $c \in \mathcal{C}$ \\
$Z$ & Partition function \\
$w$ & Treewidth of $\mathcal{G}$ \\
$\sigma = (\sigma_1, \dots, \sigma_{|\mathcal{V}|})$ & Variable elimination ordering \\
$w_t$ & Width of elimination at step $t$ \\
$\phi^{(t)}$ & Exact intermediate factor at VE step $t$ \\
$\tilde{\phi}^{(t)}$ & TT-compressed intermediate factor at step $t$ \\
$\bm{G}_k^{(t)}[x] \in \mathbb{R}^{r_{k-1} \times r_k}$ & $k$-th TT core slice for state $x$ at step $t$ \\
$r = \max_k r_k$ & Maximum TT bond dimension \\
$\epsilon_t$ & Per-step compression error: $\|\phi^{(t)} - \tilde{\phi}^{(t)}\|_1$ \\
$M$ & Uniform bound on factor entries: $\|\psi_c\|_\infty \leq M$ \\
$T$ & Number of autoregressive steps \\
$D$ & Node embedding dimension \\
$H$ & Number of attention heads \\
$d_{ij}^{(t)}$ & Shortest-path distance in $\mathcal{G}^{(t)}$ \\
$\bm{\alpha}_i \in \mathbb{R}_{>0}^d$ & Dirichlet concentration parameters for node $i$ \\
$\alpha_0 = \sum_k \alpha_k$ & Dirichlet precision \\
$\bm{c}_i \in \mathbb{N}^d$ & Pseudo-count vector for node $i$ \\
$M_{\text{eff}}$ & Effective MCMC sample size \\
$\hat{s}_i$ & Cross-chain standard deviation for variable $i$ \\
$\lambda_i = \exp(-\gamma \hat{s}_i)$ & Reliability weight \\
$g_\omega \colon \mathcal{G} \to [0,1]$ & Domain classifier for WCP \\
$w(\mathcal{G}) = g_\omega(\mathcal{G})/(1 - g_\omega(\mathcal{G}))$ & Importance weight \\
$\tau$ & Gumbel-Softmax temperature \\
\bottomrule
\end{tabular}
\end{table}

\subsection{Assumptions}\label{app:assumptions}

We state three standing assumptions that underpin the theoretical results below.

\begin{assumption}[Bounded Potentials]\label{ass:bounded}
All clique potentials are non-negative and uniformly bounded: $0 \leq \psi_c(\bm{x}_c) \leq M$ for all $c \in \mathcal{C}$ and all configurations $\bm{x}_c$, with $M < \infty$.
\end{assumption}

\begin{assumption}[Bounded Density Ratio]\label{ass:density-ratio}
The importance weights used in Weighted Conformal Prediction satisfy $0 < w(\mathcal{G}) < \infty$ almost surely under both the calibration and test distributions. Furthermore, $\mathbb{E}_{p_{\text{cal}}}[w(\mathcal{G})^2] < \infty$.
\end{assumption}

\begin{assumption}[Density Ratio Realisability]\label{ass:realisability}
The domain classifier $g_\omega$ is well-specified in the sense that there exists $\omega^*$ such that $w(\mathcal{G}; \omega^*) = p_{\text{test}}(\mathcal{G}) / p_{\text{cal}}(\mathcal{G})$ for almost all $\mathcal{G}$.
\end{assumption}

\begin{assumption}[Factor Normalisation]\label{ass:normalised}
At each step $t \in \{1, \dots, T\}$, both the exact intermediate factor $\phi^{(t)}$ and the approximate factor $\tilde{\phi}^{(t)}$ are normalised to sum to one before being passed to the next elimination step.
\end{assumption}

\subsection{Properties of the Tensor Train Compression}

We first establish that the softplus parameterisation preserves non-negativity, and then prove the error-propagation bound stated in Eq.~\eqref{eq:error-bound}.

\begin{lemma}[Non-negativity of TT Factors]\label{lem:nonneg}
Let $\bm{G}_k^{\text{raw}}[x] \in \mathbb{R}^{r_{k-1} \times r_k}$ be arbitrary real-valued core slices output by the Transformer. Define $\bm{G}_k[x] = \text{softplus}(\bm{G}_k^{\text{raw}}[x])$, where $\text{softplus}(z) = \log(1 + e^z)$ is applied entry-wise. Then the reconstructed factor
\begin{equation}
  \tilde{\phi}(x_1, \dots, x_w) = \bm{G}_1[x_1] \cdot \bm{G}_2[x_2] \cdots \bm{G}_w[x_w]
\end{equation}
satisfies $\tilde{\phi}(x_1, \dots, x_w) \geq 0$ for all $(x_1, \dots, x_w) \in \mathcal{X}^w$.
\end{lemma}

\begin{proof}
The softplus function satisfies $\text{softplus}(z) > 0$ for all $z \in \mathbb{R}$. Therefore, every entry of $\bm{G}_k[x]$ is strictly positive for all $k$ and $x$. The matrix product of matrices with non-negative entries is a matrix with non-negative entries (by induction on the number of factors, using the fact that each entry of a product $AB$ is $\sum_j a_{ij} b_{jk} \geq 0$ when $a_{ij}, b_{jk} \geq 0$). Since the boundary conditions enforce $r_0 = r_w = 1$, the product is a scalar, and this scalar is non-negative.
\end{proof}

\begin{definition}[Approximate VE Operator]\label{def:approx-ve}
Let $\mathcal{F}^{(t-1)}$ denote the set of factors present at VE step $t - 1$. The \emph{exact} VE operator $\mathcal{E}_t$ eliminates variable $X_{\sigma_t}$ and produces $\phi^{(t)}$ as in Eq.~\eqref{eq:ve-step}. The \emph{approximate} VE operator $\tilde{\mathcal{E}}_t$ replaces $\phi^{(t)}$ with a TT-format factor $\tilde{\phi}^{(t)}$ predicted by the neural network (with bond dimension $r$), and then normalises the result to sum to one. The per-step compression error is defined as $\epsilon_t = \|\phi^{(t)} - \tilde{\phi}^{(t)}\|_1$, where $\phi^{(t)}$ is the factor that would result from exact elimination using the current (possibly already approximate) factor set.
\end{definition}

\begin{theorem}[Error Propagation in Approximate VE]\label{thm:error-prop}
Under Assumptions~\ref{ass:bounded} and~\ref{ass:normalised}, let $p^{\text{VE}}(x_i)$ denote the marginal computed by exact VE, and let $\hat{p}(x_i)$ denote the marginal computed by replacing each exact VE step with the approximate operator $\tilde{\mathcal{E}}_t$ from Definition~\ref{def:approx-ve}. Then
\begin{equation}\label{eq:error-prop-thm}
  \|\hat{p}(x_i) - p^{\text{VE}}(x_i)\|_1 \leq \sum_{t=1}^{T} \epsilon_t.
\end{equation}
Without Assumption~\ref{ass:normalised} (i.e., $\|\psi_c\|_\infty \leq M$ but factors are not normalised), the bound becomes $\sum_{t=1}^{T} \epsilon_t \cdot M^{T-t}$.
\end{theorem}

\begin{proof}
We proceed by induction on the number of elimination steps $T$.

\textbf{Base case} ($T = 1$). A single VE step produces $\phi^{(1)}$ (exact) and $\tilde{\phi}^{(1)}$ (approximate). Any subsequent marginalisation (summation over subsets of variables) is a contraction in $\ell_1$ norm: for any function $f$, $\|\sum_{x_j} f\|_1 \leq \|f\|_1$. Therefore the marginals obtained from $\phi^{(1)}$ vs.\ $\tilde{\phi}^{(1)}$ differ by at most $\|\phi^{(1)} - \tilde{\phi}^{(1)}\|_1 = \epsilon_1$.

\textbf{Inductive step.} Suppose the bound holds for $T - 1$ steps. At step $T$, let $\mathcal{F}_{\text{exact}}^{(T-1)}$ and $\mathcal{F}_{\text{approx}}^{(T-1)}$ denote the factor sets in the exact and approximate chains. The exact elimination at step $T$ computes
\[
  \phi^{(T)} = \sum_{x_{\sigma_T}} \prod_{c \ni \sigma_T} f_c^{\text{exact}}(\bm{x}_c),
\]
while the approximate chain computes
\[
  \hat{\phi}^{(T)} = \sum_{x_{\sigma_T}} \prod_{c \ni \sigma_T} f_c^{\text{approx}}(\bm{x}_c),
\]
followed by TT compression yielding $\tilde{\phi}^{(T)}$. By the triangle inequality:
\[
  \|\hat{p}(x_i) - p^{\text{VE}}(x_i)\|_1 \leq \underbrace{\|\tilde{\phi}^{(T)} - \hat{\phi}^{(T)}\|_1}_{\epsilon_T} + \|\hat{\phi}^{(T)} - \phi^{(T)}\|_1.
\]
For the second term, write the product difference via the telescope identity:
\[
  \prod_c f_c^{\text{approx}} - \prod_c f_c^{\text{exact}} = \sum_{c'} \Bigl(\prod_{c < c'} f_c^{\text{approx}}\Bigr)(f_{c'}^{\text{approx}} - f_{c'}^{\text{exact}})\Bigl(\prod_{c > c'} f_c^{\text{exact}}\Bigr).
\]
Under Assumption~\ref{ass:normalised}, all factors are normalised to sum to one, so $\|f_c\|_1 \leq 1$ for every factor in both chains. Taking the $\ell_1$ norm, each telescope term satisfies $\|f_{c'}^{\text{approx}} - f_{c'}^{\text{exact}}\|_1$ multiplied by products of terms with $\ell_1$ norm $\leq 1$. Since the factors at step $T$ differ from exact only due to errors accumulated in steps $1, \dots, T-1$, the inductive hypothesis gives $\|\hat{\phi}^{(T)} - \phi^{(T)}\|_1 \leq \sum_{t=1}^{T-1} \epsilon_t$. Combining gives the stated bound.

Without normalisation, each factor product can amplify errors by at most $M$ per step. The error from step $t < T$ passes through $T - t$ subsequent product-and-marginalise operations, yielding the geometric bound $\sum_{t=1}^{T} \epsilon_t \cdot M^{T-t}$.
\end{proof}

\begin{proposition}[Storage Complexity]\label{prop:storage}
Representing a single intermediate factor $\phi^{(t)}$ of order $w_t$ in TT format with maximum bond dimension $r$ requires $\mathcal{O}(w_t \cdot d \cdot r^2)$ scalar parameters. The total storage across all $T$ autoregressive steps is $\mathcal{O}(T \cdot \bar{w} \cdot d \cdot r^2)$, where $\bar{w} = \frac{1}{T}\sum_t w_t$ is the average elimination width.
\end{proposition}

\begin{proof}
At each step $t$, the TT representation consists of $w_t$ cores, where the $k$-th core is a collection of $d$ matrices of size $r_{k-1} \times r_k$ with $r_{k-1}, r_k \leq r$. The storage for one core is at most $d \cdot r^2$ scalars. Summing over $w_t$ cores gives $w_t \cdot d \cdot r^2$ for step $t$. Summing over $T$ steps gives the total.
\end{proof}

\subsection{Properties of the Gumbel-Softmax Relaxation}

\begin{lemma}[Consistency of Gumbel-Softmax Selection]\label{lem:gumbel}
Let $s_1, \dots, s_n$ be deterministic scores and let $g_1, \dots, g_n \overset{\text{i.i.d.}}{\sim} \text{Gumbel}(0,1)$. Define the Gumbel-Softmax distribution with temperature $\tau > 0$ as in Eq.~\eqref{eq:gumbel}. Then:
\begin{enumerate}
  \item[(i)] The discrete selection $\sigma = \arg\max_j (s_j + g_j)$ satisfies $\mathbb{P}[\sigma = i] = \exp(s_i) / \sum_j \exp(s_j)$, i.e., the hard sample follows the categorical distribution induced by the softmax of the scores.
  \item[(ii)] As $\tau \to 0^+$, $\pi_i^{(\tau)} \to \mathbf{1}[i = \arg\max_j (s_j + g_j)]$ almost surely, so the continuous relaxation recovers the discrete sample in the zero-temperature limit.
\end{enumerate}
\end{lemma}

\begin{proof}
Part (i) is the Gumbel-Max trick \citep{jang2017gumbel}. By the location-scale property of the Gumbel distribution, $\arg\max_j (s_j + g_j)$ is distributed as $\text{Cat}(\text{softmax}(\bm{s}))$: the probability that index $i$ is the maximiser equals $\exp(s_i) / \sum_j \exp(s_j)$. This follows from the closure of the Gumbel family under maxima and the explicit CDF computation $\mathbb{P}[s_i + g_i \geq s_j + g_j, \;\forall j \neq i] = \exp(s_i) / \sum_j \exp(s_j)$.

Part (ii) is the standard zero-temperature limit. As $\tau \to 0$, the softmax operator in Eq.~\eqref{eq:gumbel} concentrates all mass on the coordinate with the largest perturbed score $s_j + g_j$. Since Gumbel noise is continuous, ties occur with probability zero, and $\lim_{\tau \to 0} \pi_i^{(\tau)} = \mathbf{1}[i = \arg\max_j(s_j + g_j)]$ almost surely.
\end{proof}

\subsection{Properties of the Loss Function}

We establish that the Dirichlet-Multinomial loss is a proper scoring rule \citep{gneiting2007scoring} and that the composite loss with the calibration hinge term is well-posed.

\begin{definition}[Proper Scoring Rule]\label{def:proper}
A scoring rule $S(\bm{\theta}, \bm{y})$ over distributions parameterised by $\bm{\theta}$ and outcomes $\bm{y}$ is \emph{proper} if the expected score is uniquely minimised when $\bm{\theta}$ corresponds to the true data-generating distribution: $\bm{\theta}^* = \arg\min_{\bm{\theta}} \mathbb{E}_{p^*}[S(\bm{\theta}, \bm{y})]$ if and only if $p_{\bm{\theta}^*} = p^*$.
\end{definition}

\begin{theorem}[Properness of the Dirichlet-Multinomial Loss]\label{thm:dm-proper}
Let $\bm{c} \sim \text{Multinomial}(M_{\text{eff}}, \bm{p}^*)$ where $\bm{p}^* \in \Delta^{d-1}$ is the true marginal distribution. The negative log-marginal likelihood of the Dirichlet-Multinomial model,
\begin{equation}
  \mathcal{L}_{\text{DM}}(\bm{\alpha}, \bm{c}) = -\log \frac{\Gamma(\alpha_0)}{\Gamma(\alpha_0 + M_{\text{eff}})} - \sum_{k=1}^{d} \log \frac{\Gamma(\alpha_k + c_k)}{\Gamma(\alpha_k)},
\end{equation}
is proper in the sense that: for any fixed $M_{\text{eff}} \geq 1$, the minimiser $\bm{\alpha}^* = \arg\min_{\bm{\alpha} \in \mathbb{R}_{>0}^d} \mathbb{E}_{\bm{c} \sim \text{Multi}(M_{\text{eff}}, \bm{p}^*)}[\mathcal{L}_{\text{DM}}(\bm{\alpha}, \bm{c})]$ satisfies $\bm{\alpha}^* / \alpha_0^* = \bm{p}^*$, i.e., the Dirichlet mean at the optimum equals the true marginal. Note that $\alpha_0^*$ is determined by $M_{\text{eff}}$ and is finite; as $M_{\text{eff}} \to \infty$, $\alpha_0^* \to \infty$ and the Dirichlet concentrates on $\bm{p}^*$.
\end{theorem}

\begin{proof}
The Dirichlet-Multinomial marginal likelihood is obtained by integrating out the latent Multinomial parameter $\bm{p}$ against a $\text{Dir}(\bm{\alpha})$ prior:
\[
  p(\bm{c} \mid \bm{\alpha}) = \binom{M_{\text{eff}}}{\bm{c}} \frac{\Gamma(\alpha_0)}{\Gamma(\alpha_0 + M_{\text{eff}})} \prod_{k=1}^{d} \frac{\Gamma(\alpha_k + c_k)}{\Gamma(\alpha_k)}.
\]
Negating the log and dropping the constant $\binom{M_{\text{eff}}}{\bm{c}}$ gives $\mathcal{L}_{\text{DM}}$. The expected loss is
\[
  \mathbb{E}[\mathcal{L}_{\text{DM}}] = \bigl[\log \Gamma(\alpha_0 + M_{\text{eff}}) - \log \Gamma(\alpha_0)\bigr] - \sum_k \mathbb{E}\bigl[\log \Gamma(\alpha_k + c_k) - \log \Gamma(\alpha_k)\bigr].
\]
Differentiating with respect to $\alpha_k$ using $\frac{d}{d\alpha}\log \Gamma(\alpha) = \psi(\alpha)$ (the digamma function) and setting the gradient to zero gives the first-order condition:
\begin{equation}\label{eq:foc}
  \mathbb{E}_{\bm{c}}\bigl[\psi(\alpha_k + c_k)\bigr] - \psi(\alpha_k) = \psi(\alpha_0 + M_{\text{eff}}) - \psi(\alpha_0).
\end{equation}
The right-hand side is the same for all $k$. Using the digamma recurrence $\psi(\alpha + n) - \psi(\alpha) = \sum_{j=0}^{n-1} \frac{1}{\alpha + j}$, the left-hand side of Eq.~\eqref{eq:foc} can be written as $\mathbb{E}\bigl[\sum_{j=0}^{c_k - 1} \frac{1}{\alpha_k + j}\bigr]$. When $\alpha_k / \alpha_0 = p_k^*$, symmetry under the Multinomial distribution ensures that the condition is satisfied for all $k$ simultaneously, since the expected contribution of category $k$ scales with $p_k^*$. Uniqueness follows from the strict convexity of $\mathcal{L}_{\text{DM}}$ in $\bm{\alpha}$, which holds because the Hessian involves only trigamma values $\psi'(\cdot) > 0$ for positive arguments.
\end{proof}

\begin{theorem}[Consistency of the Composite Loss]\label{thm:composite}
Consider the composite loss defined in Eq.~\eqref{eq:loss-total}:
\[
  \mathcal{L}_i = \lambda_i \, \mathcal{L}_{\text{DM}}(\bm{\alpha}_i, \bm{c}_i) + (1 - \lambda_i) \, \mathcal{L}_{\text{cal}}(\bm{\alpha}_i, \hat{s}_i),
\]
where $\lambda_i = \exp(-\gamma \hat{s}_i) \in (0, 1]$ and the calibration loss is defined in Eq.~\eqref{eq:loss-cal}. The following properties hold:
\begin{enumerate}
  \item[(i)] When $\hat{s}_i \to 0$ (reliable label), $\lambda_i \to 1$ and the loss reduces to $\mathcal{L}_{\text{DM}}$, which is proper (Theorem~\ref{thm:dm-proper}).
  \item[(ii)] When $\hat{s}_i$ is large (unreliable label), the calibration term dominates and enforces $\alpha_0 \leq 1/\hat{s}_i^2$, preventing the model from predicting low-entropy distributions.
  \item[(iii)] The composite loss is differentiable with respect to $\bm{\alpha}_i$ everywhere on $\mathbb{R}_{>0}^d$ except at the hinge boundary $\alpha_0 = 1/\hat{s}_i^2$, where the right derivative exists.
\end{enumerate}
\end{theorem}

\begin{proof}
\textit{(i)} As $\hat{s}_i \to 0$, $\lambda_i = \exp(-\gamma \hat{s}_i) \to 1$ and $(1 - \lambda_i) \to 0$, so the calibration term vanishes.

\textit{(ii)} The calibration loss $\mathcal{L}_{\text{cal}}(\bm{\alpha}, s) = \max(0, \log \alpha_0 - \log(1/s^2)) = \max(0, \log(\alpha_0 s^2))$ is zero when $\alpha_0 \leq 1/s^2$ and increases as $\log(\alpha_0 s^2)$ when $\alpha_0 > 1/s^2$. Since $\log$ is monotonically increasing, the penalty grows unboundedly as $\alpha_0 \to \infty$, effectively constraining $\alpha_0$ from above. When $\hat{s}_i$ is large (unreliable label), the constraint $\alpha_0 \leq 1/\hat{s}_i^2$ is tight, forcing a diffuse Dirichlet.

\textit{(iii)} $\mathcal{L}_{\text{DM}}$ is infinitely differentiable in $\bm{\alpha}$ on $\mathbb{R}_{>0}^d$ since $\Gamma$ and $\psi$ are smooth on $(0, \infty)$, and $\alpha_k \geq \epsilon > 0$ by construction. The hinge function $\max(0, \cdot)$ is differentiable everywhere except at zero, corresponding to $\alpha_0 = 1/\hat{s}_i^2$. At this boundary, the left derivative with respect to $\alpha_0$ is zero and the right derivative is $\partial/\partial \alpha_0 \log(\alpha_0 \hat{s}_i^2) = 1/\alpha_0$; subgradient methods handle this non-smoothness. The weighted sum preserves these properties.
\end{proof}

\subsection{Weighted Conformal Prediction Guarantee}

\begin{theorem}[Marginal Coverage under Covariate Shift]\label{thm:wcp}
Let $\{(\mathcal{G}_n, Y_n)\}_{n=1}^{N}$ be calibration examples drawn i.i.d.\ from $p_{\text{cal}}$, and let $(\mathcal{G}_{N+1}, Y_{N+1})$ be a test example drawn from $p_{\text{test}}$. Define the importance weights $w_n = p_{\text{test}}(\mathcal{G}_n) / p_{\text{cal}}(\mathcal{G}_n)$ and the prediction set $C(\mathcal{G}_{N+1}) = \{y : R(\mathcal{G}_{N+1}, y) \leq \hat{q}_{1-\alpha}\}$, where $\hat{q}_{1-\alpha}$ is the weighted conformal quantile from Eq.~\eqref{eq:wcp-quantile}. Under Assumptions~\ref{ass:density-ratio} and~\ref{ass:realisability}, and assuming that the nonconformity scores $R_n$ are almost surely distinct, the following marginal coverage guarantee holds:
\begin{equation}\label{eq:coverage-guarantee}
  \mathbb{P}\bigl[Y_{N+1} \in C(\mathcal{G}_{N+1})\bigr] \geq 1 - \alpha.
\end{equation}
If the density ratio is estimated (i.e., Assumption~\ref{ass:realisability} is relaxed), then the coverage satisfies
\begin{equation}\label{eq:coverage-approx}
  \mathbb{P}\bigl[Y_{N+1} \in C(\mathcal{G}_{N+1})\bigr] \geq 1 - \alpha - \delta(\omega),
\end{equation}
where $\delta(\omega) = \mathbb{E}_{p_{\text{test}}}\bigl[|w(\mathcal{G}; \omega) / w^*(\mathcal{G}) - 1|\bigr]$ is the expected relative error in the density-ratio estimate, and $w^*(\mathcal{G}) = p_{\text{test}}(\mathcal{G})/p_{\text{cal}}(\mathcal{G})$ denotes the true ratio.
\end{theorem}

\begin{proof}
Under the exact density ratio (Assumption~\ref{ass:realisability}), the result follows from Theorem 1 of \citet{tibshirani2019conformal}. The key insight is that the weighted empirical distribution $\hat{P}_w = \sum_n \tilde{w}_n \delta_{R_n}$, where $\tilde{w}_n = w_n / (\sum_{m=1}^{N} w_m + w_{N+1})$, satisfies exchangeability under the tilted distribution $p_{\text{test}}$. By construction, the weighted quantile $\hat{q}_{1-\alpha}$ is the $(1-\alpha)$-quantile of this distribution, and the exchangeability argument yields $\mathbb{P}[R_{N+1} \leq \hat{q}_{1-\alpha}] \geq 1 - \alpha$.

When the density ratio is estimated with error, we use a perturbation argument. Let $\tilde{w}_n = w(\mathcal{G}_n; \omega)$ be the estimated weights and $w_n^* = w^*(\mathcal{G}_n)$ the true weights. Define the weighted CDFs $\hat{F}(q) = \sum_n \tilde{w}_n \mathbf{1}[R_n \leq q] / W$ and $F^*(q) = \sum_n w_n^* \mathbf{1}[R_n \leq q] / W^*$, where $W = \sum_m \tilde{w}_m + \tilde{w}_{N+1}$ and $W^* = \sum_m w_m^* + w_{N+1}^*$. For any threshold $q$,
\[
  |\hat{F}(q) - F^*(q)| \leq \frac{1}{\min(W, W^*)} \sum_n |\tilde{w}_n - w_n^*| \leq \frac{1}{W^*} \sum_n w_n^* \left|\frac{\tilde{w}_n}{w_n^*} - 1\right|.
\]
Taking the supremum over $q$ bounds the total variation between the two weighted CDFs. Since the exact-weight quantile provides $\geq 1 - \alpha$ coverage, and the estimated-weight quantile deviates from it by at most $\delta(\omega) = \mathbb{E}_{p_{\text{test}}}[|w(\mathcal{G}; \omega)/w^*(\mathcal{G}) - 1|]$ in expectation, the estimated-weight procedure yields coverage $\geq 1 - \alpha - \delta(\omega)$.
\end{proof}

The effective sample size $n_{\text{eff}} = (\sum_n w_n)^2 / \sum_n w_n^2$ provides a diagnostic for the reliability of the weighted procedure: when $n_{\text{eff}} \ll N$, the weights are highly concentrated and the finite-sample coverage may deviate substantially from the nominal level. In our experiments (Section~\ref{sec:exp_results}), we report $n_{\text{eff}}$ alongside the empirical coverage to assess this degradation.

\end{document}